\documentclass[sigconf]{acmart} % DO NOT CHANGE THIS

\AtBeginDocument{%
  \providecommand\BibTeX{{%
    \normalfont B\kern-0.5em{\scshape i\kern-0.25em b}\kern-0.8em\TeX}}}

\usepackage{latexsym}
\usepackage{amsmath}
\usepackage{url}
\usepackage{amsmath,nccmath}
\usepackage{algorithm}
\usepackage{algorithmic}
\usepackage{xcolor}
\usepackage{graphicx}
\usepackage{subcaption}
\usepackage{balance}

\definecolor{lml}{RGB}{250, 0, 0}

\usepackage{bbm}

\newtheorem{theorem}{Theorem}
\newtheorem{lemma}{Lemma}
\theoremstyle{definition}
\newtheorem{definition}{Definition}

\usepackage{enumitem}

\DeclareMathOperator*{\argmin}{arg\,min}

\usepackage{array}
\newcolumntype{P}[1]{>{\centering\arraybackslash}p{#1}}

\newcommand{{\method}}{FairRF}
\usepackage{multirow}
\copyrightyear{2022}
\acmYear{2022}
\setcopyright{acmcopyright}\acmConference[WSDM '22]{Proceedings of the Fifteenth ACM International Conference on Web Search and Data Mining}{February 21--25, 2022}{Tempe, AZ, USA}
\acmBooktitle{Proceedings of the Fifteenth ACM International Conference on Web Search and Data Mining (WSDM '22), February 21--25, 2022, Tempe, AZ, USA}
\acmPrice{15.00}
\acmDOI{10.1145/3488560.3498493}
\acmISBN{978-1-4503-9132-0/22/02}

\title{%Towards Fair Classifiers with Unknown Sensitive Attributes \\ Debias From Where It Originates: Exploring Non-Sensitive Features for Fair Classifiers \\Exploring Inexact Sensitive Features for Fair Classifiers Without Sensitive Attributes \\Exploring Inexact Sensitive Features for Fair Classifiers \\Towards Fair Classifiers Without Sensitive Attributes: XXX Approach
Towards Fair Classifiers Without Sensitive Attributes: Exploring Biases in Related Features
}
\author{Tianxiang Zhao{\textsuperscript{\textdagger}}, Enyan Dai{\textsuperscript{\textdagger}}, Kai Shu{$^{\ddagger}$}, Suhang Wang{\textsuperscript{\textdagger}} \\}

 \affiliation{{\textsuperscript{\textdagger}}College of Information Sciences and Technology, The Pennsylvania State University, USA\\
{$^{\ddagger}$}Department of Computer Science, College of Computing, Illinois Institute of Technology%\\
 \country{USA}
 }

\email{{tkz5084, emd5759, szw494}@psu.edu, kshu@iit.edu}

\begin{document}
\fancyhead{}

\begin{abstract}
Despite the rapid development and great success of machine learning models, extensive studies have exposed their disadvantage of inheriting latent discrimination and societal bias from the training data. This phenomenon hinders their adoption on high-stake applications. Thus, many efforts have been taken for developing fair machine learning models. Most of them require that sensitive attributes are available during training to learn fair models. However, in many real-world applications, it is usually infeasible to obtain the sensitive attributes due to privacy or legal issues, which challenges existing fair-ensuring strategies. Though the sensitive attribute of each data sample is unknown, we observe that there are usually some non-sensitive features in the training data that are highly correlated with sensitive attributes, which can be used to alleviate the bias. Therefore, in this paper, we study a novel problem of exploring features that are highly correlated with sensitive attributes for learning fair and accurate classifiers. We theoretically show that by minimizing the correlation between these related features and model prediction, we can learn a fair classifier. Based on this motivation, we propose a novel framework which simultaneously uses these related features for accurate prediction and enforces fairness. In addition, the model can dynamically adjust the regularization weight of each related feature to balance its contribution on model classification and fairness. Experimental results on real-world datasets demonstrate the effectiveness of the proposed model for learning fair models with high classification accuracy.

\end{abstract}




\maketitle

\section{Introduction}
With the great improvement in performance, modern machine learning models are becoming increasingly popular and are widely used in decision-making systems such as medical diagnosis~\cite{bakator2018deep} and credit scoring~\cite{dastile2020statistical}.  
% However, the state-of-the-art architecture, deep neural network, is known to be a black box and makes opaque prediction.
% Despite their great success, recent study show that machine learning models  trained on historical data may inherit the societal bias in data, 
Despite their great successes, extensive studies~\cite{gianfrancesco2018potential,mehrabi2019survey,yapo2018ethical} have revealed that training data may include patterns of previous discrimination and societal bias. Machine learning models trained on such data can inherit the bias on sensitive attributes such as ages, genders, skin color, and regions~\cite{beutel2017data,dwork2012fairness,hardt2016equality}. 
% Therefore concerns have arise questioning its credibility and trustworthiness, worrying its negative impact once applied in the society. Hidden biases in a machine learning model could cause severe fairness problems after being deployed, especially in high-stake scenarios. 
For example, a study found strong unfairness exists in a \textit{Criminal Prediction} system used to assess a criminal defendant’s likelihood of becoming a recidivist~\cite{Julia2016machine}. The system shows a strong bias towards people with color, tending to predict them as recidivist even when they are not. Thus, hidden biases in a machine learning model could cause severe fairness problems, which raises concerns on their real-world applications, especially in high-stake scenarios.

Various efforts~\cite{feldman2015certifying,kamiran2009classifying,sattigeri2019fairness,zafar2015fairness} have been taken to address the fairness issue of current machine learning models. 
%For example, previous researches~\cite{verma2018fairness,saxena2019fairness} examine model's performance on protected groups formed according to sensitive attributes, discuss different fairness criteria and present their formalized notions. ~\cite{zhang2017achieving,kamiran2009classifying} seek to remove biases in training data via pre-processing and ~\cite{hardt2016equality,pleiss2017fairness} post-process trained model to remove unfairness. ~\cite{dwork2012fairness,zafar2015fairness} propose new optimization goals and regularization terms to remove discrimination of the model. \suhang{revise}
For example, ~\cite{kamiran2012data,feldman2015certifying} pre-process the data to remove discrimination in training.
%explore modifying the training data via perturbing those sensitive attributes.
~\cite{dwork2012fairness,zafar2015fairness} design special regularization terms to ensure that the prediction output is insensitive w.r.t sensitive attributes. And ~\cite{hardt2016equality,pleiss2017fairness} post-process prediction results on instances of unfair classes. Despite their superior performance, all the aforementioned approaches require that sensitive attributes are available for removing bias. However, for many real-world applications, it is difficult to obtain sensitive attributes of each data sample due to various reasons such as privacy and legal issues, or difficulties in data collection~\cite{coston2019fair,lahoti2020fairness}. 

Tackling fairness issue without sensitive attributes available is challenging as we lack supervision to preprocess the training data, regularize the model or post-process the predictions. There are only very few initial efforts on learning fair classifiers without sensitive attributes~\cite{lahoti2020fairness,yan2020fair,coston2019fair}. \citeauthor{yan2020fair}~\cite{yan2020fair} use a clustering algorithm to form pseudo groups to approximate real protected groups. \citeauthor{lahoti2020fairness}~\cite{lahoti2020fairness} propose to use an auxiliary module to find computationally-identifiable regions where model under-performs, and optimize this worst-case performance. However, these works are often found to be ineffective in achieving fairness with demographics~\cite{lahoti2020fairness}. In addition, the groups or regions found by these approaches may not be related to the sensitive attribute we want to be fair with. For example, we might want the model to be fair on \textit{gender}; while the clustering algorithm gives groups of \textit{race}. %Lacking of information on protected groups limits their ability. 
Thus, more efforts need to be taken to address the important and challenging problem of learning fair models without sensitive attributes.

%\textit{Though the sensitive attribute of each data sample is unknown, can we find other external prior knowledge to guide the training of fair machine learning models?} 
Though the sensitive attribute of each data sample is unknown, we observe that \textit{there are usually some non-sensitive features in the training data that are highly correlated with sensitive attributes, which can be used to alleviate the bias}.
Previous works~\cite{Julia2016machine,coston2019fair} observed that unfairness persists even when sensitive attributes are not used as input, which indicate that biases are embedded in some non-sensitive features used for training models. These non-sensitive features are highly correlated with sensitive attributes, which makes the model biased. We call such  features as \textit{Related Features}.
%Inspired by this observation, we dedicate to answer this question by exploiting the given attributes.  its elements (denoted as related attributes) that are correlated with sensitive attributes. T
These correlations arise from various reasons, such as biases in data collection, or interplay of an underlying physiological difference with socially determined role perception~\cite{celentano1990gender}. 
For example, \citeauthor{vogel2016toward}~\cite{vogel2016toward} find that there exist striking differences in age distributions across racial/ethnic groups in US prisons. The Hispanic and black populations have a larger portion of individuals at younger ages, hence age is correlated with race in this field. In practice, common sense and prior domain knowledge can help to identify the related features given that we want to have a fair model on certain sensitive attributes. In addition, for different sensitive attribute such as \textit{race} or \textit{gender}, we can specify different sets of related features. With these related features identified, we would be able to alleviate the fairness issue. One straightforward way is to discard related features for training a fair model. However, it will also discard important information for classification. Thus, though promising, it remains an open question of how to effectively utilize related features to learn fair models with high classification accuracy.

%according to the U.S. Bureau of Justice Statistics(BJS), in 2018 black males accounted for $34\%$ of the total male prison population, while black females accounted for only $18\%$ of the female population. White males accounts for $29\%$ of male prisoners while white females accounts for $47\%$ of female prisoners. Hence, we can see that gender is correlated with race in this field. {\tianxiang{change to use education for race}}

%For example, ~\cite{nocelentano1990gender, nolen2009gender, wetter1999gender, essau2010gender} found that statistical differences in the develop course of specific illness can be observed across genders, like symptom frequency, associated pain and duration. Due to interplay of an underlying physiological difference with socially determined role perception~\cite{}, differences in social preference(like sociability) have also been observed cross gender or race~\cite{}.

Therefore, in this paper, we study a novel problem of exploring related features for learning fair and accurate classifiers without sensitive attributes. In essence, we are faced with three challenges: (i) how to utilize these related features to achieve fairness; (ii) how to achieve an optimal trade-off between accuracy and fairness; (iii) when given related feature sets contains misidentified features or are incomplete, how to adjust the usage of them. In an attempt to solve these challenges, we propose a novel framework \underline{F}airness with \underline{R}elated \underline{F}eatures ({\method}). Instead of simply discarding related features, the basic idea of {\method} is to use the related features as both features for training the classifier and as pseudo sensitive attributes to regularize the behavior of it, which help to learn fair and accurate classifiers. We theoretically show that regularizing the model using related features can achieve fairness on sensitive attribute. % We provide theoretical analysis on the effects of using related attributes as surrogates, which proves its benefits. % Then regularization is implemented on them, with related weights to make a balance between constraining fairness and performing classification. 
Furthermore, to balance the classification accuracy and model fairness, and cope with the case when identified related attributes are inaccurate and noisy, {\method}  can automatically learn the importance weight of each related feature for regularization in the model.  % {\method}as precise related weights are difficult to obtain,
%we design an algorithm to dynamically learn the importance with model parameters. % Extensive experiments are conducted to evaluate this proposal, and various strategies in selecting those attributes are tested.
The main contributions of the paper are as follows:
\begin{itemize}
    \item We study a novel problem of exploring  related features to learn fair classifiers without sensitive attributes;
    \item We theoretically show that by adopting related features to regularize the model, we can learn fairer classifier; % We theoretically prove the benefits of enforcing constraints on related attributes, and design an algorithm implementing that. Furthermore, we extends proposed method to fine-tune the weights of related attributes, making it easier to be applied in real-world problems;
    \item We propose a novel framework {\method} which can simultaneously utilize the related features to learn fair classifiers and adjust the importance weights of each related feature; and
    \item We conduct extensive experiments on real-world datasets to demonstrate the effectiveness of the proposed method for fair classifiers with high classification accuracy.%   are conducted, which empirical validate this proposal. Proposed method outperforms baseline approaches with a clear margin.
\end{itemize}

%The rest of the paper are organized as follows. In Sec.~\ref{sec:related_work}, we review related works. In Sec.~\ref{sec:problem_definition}, we formally define the problem. In Sec.~\ref{sec:prelimi}, we provide theoretical foundation of {\method}. In Section~\ref{sec:methodology}, we give the details of {\method}. In Sec.~\ref{sec:optimization_algorithm} we describe an optimization framework for {\method}. In Sec.~\ref{sec:experiment}, we conduct experiments to evaluate the effectiveness of {\method} followed by the conclusion in Sec.~\ref{sec:conclusion}.

\section{Related Work} \label{sec:related_work}
% \subsection{Fairness in Machine Learning}
To address the concerns of fairness in machine learning models, a number of fairness approaches are proposed. They can be generally split into three categories: (i) individual fairness~\cite{dwork2012fairness,zemel2013learning,kang2020inform,lahoti2019operationalizing}, which requires the model to give similar prediction to similar individuals; (ii) group fairness~\cite{dwork2012fairness,hardt2016equality,zhang2017achieving}, which aims to treat the groups with different protected sensitive attributes equally; (iii) Max-Min fairness~\cite{lahoti2020fairness,hashimoto2018fairness,zhang2014fairness}, which tries to maximize the minimum expected utility across groups. We focus on group fairness in this work.

Extensive works have been conducted to for group fairness-aware machine learning~\cite{zhang2017achieving,beutel2017data,locatello2019fairness,dwork2012fairness,hardt2016equality,zemel2013learning,lahoti2020fairness}. Based on the stage of applying fairness in training, these algorithms can be generally split into three categories: pre-processing approaches~\cite{zhang2017achieving,kamiran2012data,xu2018fairgan}, in-processing approaches~\cite{zafar2015fairness,zhang2018mitigating}, and post-processing approaches~\cite{hardt2016equality,pleiss2017fairness}. 
Pre-processing approaches modify the training data to reduce the historical discrimination in the dataset. For instance, the bias could be eliminated by correcting labels \cite{zhang2017achieving,kamiran2009classifying}, revising attributes \cite{kamiran2012data,feldman2015certifying}, generating non-discriminatory data \cite{xu2018fairgan,sattigeri2019fairness}, and obtaining fair representations \cite{beutel2017data,locatello2019fairness,edwards2015censoring,zemel2013learning,louizos2015variational,creager2019flexibly}. In-processing approaches revise the training of the state-of-the-art models to achieve fairness. 
More specifically, they apply fairness constraints or design a objective function considering the fairness of predictions~\cite{dwork2012fairness,zafar2015fairness,zhang2018mitigating}. Finally, the post-processing approaches directly change the predictive labels of trained models to obtain fair predictions~\cite{hardt2016equality,pleiss2017fairness}. %These fair-learning methods have been successfully applied in a range of real-world scenarios, including \textit{Criminal Prediction}~\cite{Julia2016machine}, \textit{Recommendation System}~\cite{burke2017multisided,beutel2019fairness},  \textit{Fraud Detection}~\cite{shekhar2021fairod}, etc.

Despite their ability in alleviating the bias issues, aforementioned methods generally require the sensitive attributes of each data sample available to achieve fairness; while for many real-world applications,  it is difficult to collect sensitive attributes of subjects due to various reasons such as privacy issues, legal problems and regulatory restrictions. The lacking of sensitive attributes of training data challenges the aforementioned methods~\cite{beutel2017data}. Investigating fair models without sensitive attributes is important and challenging, and it is still in its early stage. There are only a few works on this direction~\cite{lahoti2020fairness,hashimoto2018fairness,yan2020fair}. One branch of approaches~\cite{lahoti2020fairness,hashimoto2018fairness} investigates fairness without demographics via solving a Max-Min problem. For instance, \citeauthor{lahoti2020fairness}~\cite{lahoti2020fairness} proposes adversarial reweighted learning that leverages the notion of computationally-identifiable errors to achieve Rawlsian Max-Min fairness without sensitive attributes. However, these methods are only effective for achieving Max-Min fairness. %and are ineffective for group fairness. 
The other branch~\cite{dai2020fairgnn,yan2020fair} addresses this missing sensitive attribute scenario via providing pseudo group splits. %For instance, \citeauthor{dai2020fairgnn}~\cite{dai2020fairgnn} designs an estimator to predict sensitive attributes. 
For instance, \citeauthor{yan2020fair}~\cite{yan2020fair} pre-processes the data via clustering and uses obtained groups as the proxy. However, the conformity between obtained groups from these approaches and real protected groups are highly dependent on data distribution.%, which makes it difficult to justify generalization ability of them.  
%Some individual fairness methods~\cite{dwork2012fairness,kang2020inform} do not need sensitive attributes. But they need a metric designed to determine the similar instances, which is often impractical.

The proposed {\method} is inherently different from the aforementioned approaches: (i) We study a novel problem of exploring features that are highly related to the unseen sensitive ones for learning fair and accurate classifiers. Obtaining these features requires just a little prior domain knowledge, and it prevents the difficulty and instability of previous approaches in detecting protected groups~\cite{lahoti2020fairness,yan2020fair}; and (ii) We theoretically show that by regularizing the model prediction with the related features that are highly corrected with sensitive attributes, we can learn a fair model w.r.t the sensitive attribute. In addition, our experimental results show that the given related feature set can be incomplete or noisy.

% Recently, several works explore the learning of fair graph embeddings for recommendation \cite{rahman2019fairwalk,bose2019compositional}. Fairwalk \cite{rahman2019fairwalk} modifies the random walk procedure of node2vec \cite{grover2016node2vec} to obtain a more diverse network neighborhood representations. The sensitive attributes of all the nodes are required in the sampling procedure of FairWalk. \citeauthor{bose2019compositional} \cite{bose2019compositional} propose to add discriminators to eliminate the sensitive information in the graph embeddings. Similar to Fairwalk, the training process of the discriminators is in need of the sensitive attributes of all the nodes.

\section{Problem Definition} \label{sec:problem_definition}
% Before formally defining the problem, we first introduce notations used in this paper. 
Throughout this paper, matrices are written as boldface capital letters and vectors are denoted as boldface lowercase letters. For an arbitrary matrix $\mathbf{M} \in \mathbb{R}^{n \times m}$, $M_{ij}$ denotes the $(i, j)$-th entry of $\mathbf{M}$ while $\mathbf{m}_i$ and $\mathbf{m}^j$ mean the $i$-th row and $j$-th column of $\mathbf{M}$, respectively. Capital letters in calligraphic math font such as $\mathcal{P}$ are used to denote sets or cost function.

Let $\mathbf{X} \in \mathbb{R}^{n \times m}$ be the data matrix with each row $\mathbf{x}_i \in \mathbb{R}^{1 \times m}$ as an $m$-dimensional data instance. We use $\mathcal{F}=\{f_1,\dots,f_m\}$ to
denote the $m$ features and $\mathbf{x}^1, \dots,\mathbf{x}^{m}$ are the corresponding feature vectors, where $\mathbf{x}^{j}$ is the $j$-th column of $\mathbf{X}$.  Let $\mathbf{y} \in \mathbb{R}^{n}$ be the label vector, where the $i$-th element of $\mathbf{y}$, i.e., $y_i$, is the label of $\mathbf{x}_i$. Following existing work on fair machine learning models~\cite{lahoti2020fairness}, we focus on binary classification problem, i.e., $y_i \in \{0,1\}$. Given $\mathbf{X}$ and $\mathbf{y}$, we aim to train a fair classifier with good classification performance. 

Extensive studies~\cite{Julia2016machine,lahoti2020fairness} have revealed that historical data may include previous discrimination and societal bias on sensitive attribute $S$ such as ages, genders, skin color, and regions. Though sensitive attributes $S$ are not used as features, i.e., $S \not\in \mathcal{F}$, a subset of none-sensitive features $\mathcal{F}_s \in \mathcal{F}$ are highly correlated with sensitive attributes, making machine learning models trained on such data inherit the bias. For example, in dataset containing US criminal records~\cite{Julia2016machine}, racial information is taken as sensitive. Although it is unseen, trained model could still be unfair as distribution of racial groups population may be leaked from the distribution of ages~\cite{vogel2016toward}.

In many real-world applications, sensitive attributes of data samples are unavailable due to various reasons such as difficulty in data collection, security or privacy issues. It challenges existing fair machine leaning approaches that require sensitive attributes of data samples for fair models. Though sensitive attribute of each data sample is unknown, since the bias is caused by the subset of features $\mathcal{F}_S$ that are highly correlated with $S$, $\mathcal{F}_S$ can provide alternative supervision to learn fair models. Therefore, we aim to explore the utilization of $\mathcal{F}_{S}$ to help learn more fair model meanwhile maintain high classification performance. The problem is defined as:

\vspace*{0.5em}
\noindent{}\textbf{Problem Definition} Given the data matrix $\mathbf{X}\in \mathbb{R}^{n \times m}$, with corresponding labels $\mathbf{y} \in \mathbb{R}^{n}$, and a predefined feature subset $\mathcal{F}_{S} \in \mathcal{F}$, where each $f_i \in \mathcal{F}$ called related feature which highly correlates with the unobserved protected attribute $S$, \textit{e.g., race or gender}, learn a classifier that maintains high accuracy and is fair on $S$.
\vspace*{0.5em}

Note that we assume $\mathcal{F}_{S} \in \mathcal{F}$ is given from domain knowledge or experts. In practice, $\mathcal{F}_{S}$ can be incomplete and noisy. We design {\method} that is able to re-weight each $f_i \in \mathcal{F}_{S}$, so that it has the potential of remaining effective, as shown in experiments.

\section{Preliminary Theoretical Analysis}\label{sec:prelimi}

%\subsubsection{Propagation Property of Correlation}
In this paper, we adopt Pearson correlation coefficient to measure the correlation between two variables, defined as below:
 % Correlation coefficient normalizes the correlation to range between $[-1.0, 1.0]$, making the analysis simpler.

\theoremstyle{definition}
\begin{definition}[Pearson Correlation Coefficient]
% Correlation coefficient is a statistical measure of the strength of relationship between two random variables. In this work, we adopt 
Pearson correlation coefficient measures the linear correlation between two random variables $X$ and $Y$ as:
\begin{equation}
\label{equation:CC}
    \rho_{X, Y} = \frac{\mathbb{E}[(X-\mu_{X}) \cdot (Y-\mu_{Y})]}{\sigma_{X} \cdot \sigma_{Y}},
\end{equation}
where $\mu_X$ is the mean and $\sigma_{X}$ is the standard deviation of $X$.
\end{definition}

% Before presenting our method, in this section 
Next, we will show a theorem on the propagation property of Pearson correlation coefficient, which justifies our motivation of using $\mathcal{F}_{S}$ to regularize model predictions in the case of absent $S$. 
Below, we first present a rule depicting the relation of three included angles in space, which is the basis of our proof. 

\begin{lemma}
\label{lemma:angle}
Given a unit sphere centered at origin $O$, $A, B$ and $C$ are three points on the surface of the sphere. Assume that the angle $AOB=\theta_1$ and the angle $BOC=\theta_2$, then the cosine value of angle $AOC$ is within: $[cos(\theta_1+\theta_2), cos(\theta_1-\theta_2)]$.
\end{lemma}

\begin{proof}
From Spherical law of cosines~\cite{gellert2012vnr}, we can know that:
\begin{equation}
    cos \theta_3 = cos\theta_1 cos\theta_2 + sin\theta_1sin\theta_2cosB',
\end{equation}
where $B'$ corresponds to the angle opposites $B$ in spherical triangle $ABC$. As all angles are in the scale $[0, \pi]$, we can directly induce:
\begin{equation}
\begin{aligned}
   cos \theta_3 & \ge cos\theta_1 cos\theta_2 - sin\theta_1sin\theta_2  = cos(\theta_1 + \theta_2) \\ 
   cos \theta_3 & \leq cos\theta_1 cos\theta_2 + sin\theta_1sin\theta_2  = cos(\theta_1 - \theta_2),
\end{aligned}
\end{equation}
which completes the proof.
%As cosine is a monotonically decreasing function in $[0, \pi]$, we have that $\theta_3 \in [|\theta_1-\theta_2|, min(\theta_1+\theta_2, 2\pi-\theta_1-\theta_2)]$, which finishes the proof.
\end{proof}

Next, we will show the relationship between Pearson correlation coefficient and cosine similarity of two variables.
\begin{lemma}
\label{lemma:CosCorre}
Given two random variables $X, Y$, Pearson correlation coefficient between them can be calculated as the cosine distance between $\mathbf{x}'$ and $\mathbf{y}'$, where $\mathbf{x}'$ is an infinite-length vector constructed by sampling z-score value of $X$, i.e., $x'_i = \frac{X_i - \mu_{X}}{\sigma_X}$ and $X_i$ is the $i$-th sample. Similarly, $y'_i = \frac{y_i - \mu_{Y}}{\sigma_Y}$.
\end{lemma}
\begin{proof}
This can be easily proven by re-writing the form of Pearson correlation coefficient as:
\begin{equation}
\small
\begin{aligned}
    \rho_{X,Y} & = \frac{\mathbb{E}[(X-\mu_{X}) \cdot (Y-\mu_{y})]}{\sigma_{X} \cdot \sigma_{Y}}  %= \lim_{n\to\infty} \frac{\sum_{i=1}^{n} (X_i-\mu_{X}) \cdot (Y_i-\mu_{y})}{n \cdot \sigma_{X} \cdot \sigma_{Y}} \\
    = \lim_{n\to\infty} \sum_{i=1}^{n} \frac{(X_i-\mu_{X}) \cdot (Y_i-\mu_{y})}{\sigma_{X} \cdot \sigma_{Y}} \\
    & = \lim_{n\to\infty} \sum_{i=1}^{n} x_i' \cdot y_i' = cos(\mathbf{x}', \mathbf{y}'),
\end{aligned}
\end{equation}
% \begin{equation}
% \begin{aligned}
%     \rho_{X,Y} & = \frac{\mathbb{E}[(X-\mu_{X}) \cdot (Y-\mu_{y})]}{\sigma_{X} \cdot \sigma_{Y}} \\
%     & = \lim_{n\to\infty} \frac{\sum_{i=1}^{n} (X_i-\mu_{X}) \cdot (Y_i-\mu_{y})}{n \cdot \sigma_{X} \cdot \sigma_{Y}} \\
%     & = \lim_{n\to\infty} \sum_{i=1}^{n} \frac{(X_i-\mu_{X}) \cdot (Y_i-\mu_{y})}{\sigma_{X} \cdot \sigma_{Y}} \\
%     & = \lim_{n\to\infty} \sum_{i=1}^{n} X_i' \cdot Y_i' = cos(X', Y'),
% \end{aligned}
% \end{equation}
which completes the proof.
\end{proof}

With these preparations, we can now turn to our main theorem:
\begin{theorem}
\label{theorem:1}
Given three random variables $\{X, Y, Z\}$, with correlation coefficient $\rho_{X,Y} = cos \alpha$ and $\rho_{Y,Z} = cos \beta$, $\alpha, \beta \in [0, \pi]$,  then $\rho_{X, Z}$ is within $[cos(\alpha+\beta), cos(\alpha-\beta)]$. 
\end{theorem}

\begin{proof}
The proof can be developed via the following steps:
\begin{enumerate}[leftmargin=*]
    \item Cosine similarity between $\mathbf{x}'$ and $\mathbf{y}'$ shows the cosine value of included angle between them. Hence, based on Lemma~\ref{lemma:CosCorre}, we can learn that the cosine of angle between $\mathbf{x}'$ and $\mathbf{y}'$ is $cos(\alpha)$ and that of angle between $\mathbf{y}'$ and $\mathbf{z}'$ is $cos(\beta)$ from the given correlation coefficients. %\suhang{notation: cos}
    \item $\mathbf{x}'$, $\mathbf{y}'$, $\mathbf{z}'$ can be taken as line $OA$, $OB$, $OC$ in Lemma~\ref{lemma:angle} respectively. Hence, utilizing Lemma~\ref{lemma:angle}, we could induce that the cosine value of angle $\gamma$ between $\mathbf{x}$ and $\mathbf{z}$ should fall within the scale $[cos(\alpha+\beta), cos(\alpha-\beta)]$.
    \item Finally, based on Lemma\ref{lemma:CosCorre}, we can map the cosine value of angle $\gamma$ back into correlation coefficient between $X$ and $Z$.
\end{enumerate}
After these steps, we can obtain that $\rho_{X, Z}=cos(\mathbf{x}', \mathbf{z}') \in [cos(\alpha+\beta), cos(\alpha-\beta)]$ and finish the proof.
\end{proof} 

%\subsubsection{Foundation of }

Basing on theorem~\ref{theorem:1}, we can show how the constraint of correlation scale is propagated from $\mathcal{F}_{S}$ to $S$ in Theorem~\ref{theorem:2}, which theoretically proves our idea.  
\begin{theorem}
\label{theorem:2}
Let $f$ and $S$ represent an input feature and sensitive attribute, respectively. Let $\hat{y}$ denotes the variable of model's prediction. Assume that $f$ is highly correlated with $S$, i.e., $\rho_{f,S}$ is larger than a positive constant $cos \alpha$. If the model is trained to make $\rho_{f,\hat{y}}$ near $0$, i,e, within $[cos (\frac{1}{2}\pi+\delta), cos (\frac{1}{2}\pi-\delta)]$, where $\delta$ is close to $0$,  then $\rho_{S,\hat{y}}$ would be within $[cos (\frac{1}{2}\pi+\delta+\alpha), cos (\frac{1}{2}\pi-\delta-\alpha) ]$. 
\end{theorem}

Theorem~\ref{theorem:2} can be easily proved based on Theorem~\ref{theorem:1}. From it, we can see that when $cos\alpha \approx 1$ and $\delta \approx 0$, $\rho_{S,\hat{y}}$ would also approximate $0$. In this way, the prediction would be insensitive towards $S$, achieving fairness w.r.t sensitive attribute $S$.

%Assume that $f_j \in \mathcal{F}_S$ is highly correlated with $S$, i.e., $\rho_{f_j,S}$ is larger than a positive constant $cos \alpha$. Let $\hat{Y}$ denote the variable of model's prediction. Based on Theorem~\ref{theorem:1}, if the model can be trained to make $\rho_{f_j,\hat{Y}}$ near $0$, i,e, within $[cos (\frac{1}{2}\pi+\delta), cos (\frac{1}{2}\pi-\delta)]$,  then $\rho_{S,\hat{Y}}$ would be within $[cos (\frac{1}{2}\pi+\delta+\alpha), cos (\frac{1}{2}\pi-\delta-\alpha) ]$, where $\delta$ is close to $0$. 

We can extend Theorem~\ref{theorem:2} to the case of utilizing multiple related features simultaneously. For a set of related features $\mathcal{F}_S=\{f_1, f_2, ..., f_K\}$, assume their correlation coefficient with $S$ in the form of $\{ cos\alpha_1, cos\alpha_2, ..., cos\alpha_K \}$, and with $\hat{Y}$ in the range of $[cos (\frac{1}{2}\pi+\delta), cos (\frac{1}{2}\pi-\delta)]$. Then $\rho_{S,\hat{Y}}$ would fall upon the intersections of their resulting value space, which can be written as: %\suhang{ same $\delta$?}
\begin{equation}
    \rho_{S,\hat{Y}} \in [cos (\frac{1}{2}\pi+\delta+\alpha_{min}), cos (\frac{1}{2}\pi-\delta-\alpha_{min}) ].
\end{equation}
where $\alpha_{min}$ is the smallest value in $\{ \alpha_1,  \alpha_2, ...,  \alpha_K\}$. Note that this range is usually not tight, and high divergence within $\mathcal{F}_S$ would often restrict the range of $\rho_{S, \hat{Y}}$ more.

\begin{figure}
  \centering
    \includegraphics[width=0.5\textwidth]{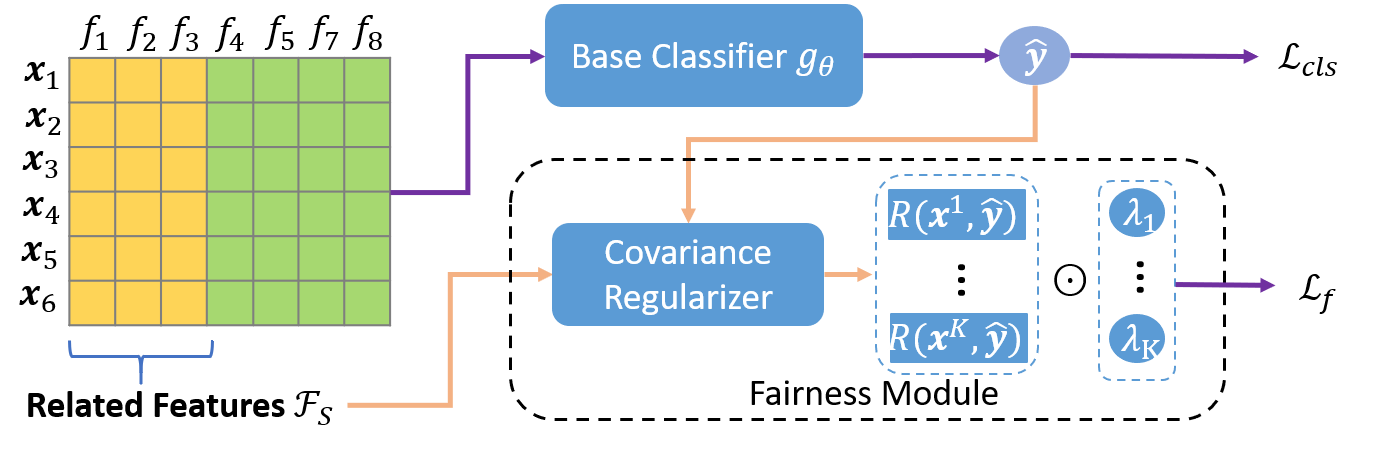}
    \vskip -1.5em
    \caption{An illustration of the proposed framework {\method}. In Fairness Constraint block, $\lambda_i$ controls the importance of regularization on $i$-th feature of $\mathcal{F}_S$.  $\boldsymbol\lambda$ is dynamically updated, reducing prior domain knowledge required. }\label{fig:idea}
\end{figure}

\section{Methodology} \label{sec:methodology}
In this section, we present the details of the proposed framework {\method}. The basic idea is using the regularization on correlated features $\mathcal{F}_S$ as the surrogate fairness objective. With the motivation theoretically justified in Sec~\ref{sec:prelimi}, an illustration of {\method} is shown in Figure~\ref{fig:idea}. It is composed of three parts: (i) a base classifier $g_{\theta}(\cdot)$ which predicts its label $\hat{y}_i$ given data sample $\mathbf{x}_i$; (ii) a covariance regularizer which constrains correlation between $\mathcal{F}_{S}$ and  $\hat{y}$ to achieve fairness; and (iii) an importance learning module which adjusts importance score $\lambda_j$ of each related feature $f_j \in \mathcal{F}_S$. Next, we introduce each component in detail. % An illustration of the overall idea can be seen in Figure~\ref{fig:idea}. \suhang{revise}

\subsection{Base Classifier}
The proposed {\method} is flexible to use various classifiers as backbone such as neural networks, logistic regression and SVM. Without loss of generality, we use $g_\theta(\cdot)$ to denote the base classifier, where $\theta$ is the set of parameters of the base classifier.  Following existing work on fairness~\cite{lahoti2020fairness}, we consider binary classification. We leave the extension to multi-class classification as future work. For a data sample $\mathbf{x}_i$, the predicted probability of $\mathbf{x}_i$ having label $1$ is
\begin{equation}
    \hat{y}_i = g_{\theta}(\mathbf{x}_i)
\end{equation}
Then the binary cross entropy loss for training the classifier $g_\theta(\cdot)$ can be written as
\begin{equation}
    \min_{\theta} \mathcal{L}_{cls} = \sum_{i=1}^n -y_i \log \hat{y}_i - (1 - y_i) \log (1 - \hat{y}_i)
\end{equation}
where $y_i \in \{0,1\}$ is the label of $\mathbf{x}_i$.

Generally, the well trained model is good at classification. However, as shown in previous studies~\cite{zhang2017achieving,beutel2017data}, the obtained model could make unfair predictions because spurious correlation may exist in the training data between sensitive attributes and labels due to societal bias. Though various efforts have been taken to mitigate the bias~\cite{dwork2012fairness,hardt2016equality,zafar2015fairness}, most of them require knowing the sensitive attributes. With the sensitive attributes unknown, to learn fair models, we propose to regularize the predictions using the related features $\mathcal{F}_S$ that are highly correlated with $S$, which will be introduced next.% Addressing it requires specially-designed countermeasures, which will be introduced below.

\subsection{Exploring Related Features for Fairness}
% As shown in Problem Definition, in our problem setting we assume that input feature $\mathbf{F}$ is dependent on both sensitive attribute $S$ and the real causal factor.

If the sensitive attribute $s_i$ of each data sample $\mathbf{x}_i$ is known, we can adopt $s_i$ to achieve fairness of the classification model by making the prediction independent of the sensitive attributes~\cite{dwork2012fairness,zafar2015fairness}. Let $\mathbf{s} \in \mathbb{R}^{n \times 1}$ be the sensitive attribute vector with the $i$-th element of $\mathbf{s}$, i.e., $s_i$, as the sensitive attribute of $\mathbf{x}_i$. Similarly, let $\hat{\mathbf{y}} \in \mathbb{R}^{n \times 1}$ be the predictions with the $i$-th element being the prediction for $\mathbf{x}_i$. Following the design in ~\cite{zafar2015fairness,dai2020fairgnn}, the pursuit of non-dependence between prediction $\hat{y}$ and sensitive attribute $\mathbf{s}$ can be achieved through minimizing the correlation score between them, which can be mathematically written as: %Without loss of generality, we show the form of this regularization term in batches:
% \begin{equation}
%     \mathcal{R}(\mathbf{s}, \hat{\mathbf{y}}) = \Bigg{|}\sum_{i=1}^{b} \Big{[}(s_{i}-\frac{\sum_{j=1}^{b} S_{j}}{b}) \cdot (\hat{y}_{i}-\frac{\sum_{j=1}^{b} \hat{y}_{j}}{b})\Big{]}\Bigg{|},
% \end{equation}
\begin{equation}
    \min_{\theta} \mathcal{R}(\mathbf{s}, \hat{\mathbf{y}}) = \Big{|}\sum_{i=1}^{n}(s_i - \mu_{s})(\hat{y}_i - \mu_{\hat{y}})\Big{|}
\end{equation}
where $\mu_s$ and $\mu_{\hat{y}}$ are the mean of $\mathbf{s}$ and $\hat{\mathbf{y}}$, respectively.
%in which $S_i$ and $\hat{y}_i$ corresponds to the sensitive attribute and predicted positive probability respectively for the $i$-th instance. $b$ is the number of instances within a batch. 
Note that we set constraints directly on the correlation score instead of correlation coefficient, but it can be seen from Eq.\ref{equation:CC} that it only differs from correlation coefficient by a constant multiplier $\sigma_{\mathbf{s}} \cdot \sigma_{\hat{y}}$. %When $b$ approaches infinite, it would be exactly the same as calculated Pearson's correlation score. 
Constraining the scale of this regularization term, $\mathbf{s}$ and $\hat{y}$ would be encouraged to have no statistical correlation with each other.

\textit{However, as sensitive attribute $\mathbf{s}$ is unavailable in our problem, directly adopting the above regularization is impossible}. Fortunately, from Theorem~\ref{theorem:2}, we can see that if we have a set of non-sensitive features $\mathcal{F}_s$, with each feature $f_j \in \mathcal{F}_s$, i.e., $\mathbf{x}^{j}$, having high correlation with $\mathbf{s}$,  then reducing the correlation between $\mathbf{x}^{j}$ with $\hat{\mathbf{y}}$ can indirectly reduce the correlation between $\mathbf{s}$ and $\hat{\mathbf{y}}$, which helps to achieve fairness, even though $\mathbf{s}$ is unknown. Hence, in {\method}, we apply correlation regularization on each feature $f_j \in \mathcal{F}_S$, in the purpose of making trained model fair towards $S$. Without loss of generality, let the set of features in $\mathcal{F}_S$ be $\{f_1,\dots,f_K\}$, where $1 \le K <m$. %Assume we have $K$ different attributes in $\mathbf{F}^S$, the adopted regularization can be written as:
The regularization term is written as
\begin{equation}
    \min_{\theta} \mathcal{R}_{related} = {\sum}_{j=1}^{K} \lambda_j \cdot \mathcal{R}(\mathbf{x}^j, \hat{\mathbf{y}}),
\end{equation}
where $\lambda_j$ is the weight for regularizing correlation coefficient between $\mathbf{x}^j$ and $\hat{\mathbf{y}}$.  $\mathcal{R}(\mathbf{x}^j, \hat{\mathbf{y}})$ is given as
\begin{equation}
    \mathcal{R}(\mathbf{x}^j, \hat{\mathbf{y}}) = \Big{|}\sum_{i=1}^{n}(X_{ij} - \mu_{x^j})(\hat{y}_i - \mu_{\hat{y}})\Big{|}
\end{equation}
where $\mu_{x^j}$ is the mean of $\mathbf{x}^j$.

% As can be seen from Theorem~\ref{theorem:1}, 
Generally, if the correlation between $\mathbf{x}^j$ and $\mathbf{s}$ is large, we would prefer large $\lambda_j$ to enforce $\mathcal{R}(\mathbf{x}^j, \hat{\mathbf{y}})$ to be close to $0$, which can better reduce the correlation between $\mathbf{s}$ and $\hat{\mathbf{y}}$, resulting in a more fair classifier. If the correlation between $\mathbf{x}^j$ and $\mathbf{s}$ is not that large, a small $\lambda_j$ is preferred because under such case, making $\mathcal{R}(\mathbf{x}^j, \hat{\mathbf{y}})$ close to $0$ doesn't help much in making $\mathbf{s}$ and $\hat{\mathbf{y}}$ independent, but may instead introduce large noise in label prediction. Domain knowledge would be helpful in setting $\lambda_j$. %  However, as sensitive attribute is unknown, the fairness of trained prediction model would be dependent on the criteria of selecting related attributes. It can be seen from Theorem~\ref{theorem:1} that selected $\mathbf{F}^S$ is preferred to have high correlation with $S$. In this work, we set them basing on external prior domain knowledge. 

\subsection{Learning Importance of Related Features}
One limitation of this approach is the requirement of pre-defined $\boldsymbol{\lambda}$. 
This information provides prior knowledge and is important for the success of the proposed proxy regularization. However, in real-world applications, it is difficult to get accurate values, and $\mathcal{F}_{S}$ could be inaccurate. In addition, \textit{$\lambda_j$ is also important in balancing the contribution of $f_j$ in model prediction and fairness.} Larger $\lambda_j$ will result in the independence between $\mathbf{x}^j$ and $\hat{\mathbf{y}}$, making $f_j$ contributes little in model prediction.  %which might significantly affect the prediction accuracy if $f_j$ is important in model accuracy.   
Hence, in this section, we propose to learn $\boldsymbol{\lambda}$, allowing the model to automatically adjust its value.

Specifically, before learning, each related weight $\lambda_j$ is initialized to a pre-defined value $\lambda_j^0$, which serves as an inaccurate estimation of its importance. Then, during training, the value of $\boldsymbol{\lambda}$ will be optimized along with model parameters iteratively. As no other information is available, we update $\boldsymbol{\lambda}$ by minimizing the total regularization loss, based on the intuition that an ideal surrogate correlation regularization should be achieved without causing significant performance drop. We limit the range of $\boldsymbol{\lambda}$ as $[0, 1]$, and the full optimization objective function can be written as follows: 
\begin{equation}
    \min_{\boldsymbol{\theta}, \boldsymbol{\lambda}} ~ \mathcal{L}_{cls} + \eta \cdot {\sum}_{j=1}^{K} \lambda_j \cdot \mathcal{R}(\mathbf{x}^{j}, \hat{\mathbf{y}}) %+ \beta \|\boldsymbol{\lambda}\|_2^2 \\ 
\begin{aligned} \label{eq:obj_1}
    \text{\quad s.t. \quad}   \lambda_j \ge 0,  \forall f_j \in \mathcal{F}_S; ~ {\sum}_{j=1}^{K} \lambda_j = 1
\end{aligned}
\end{equation}
% \begin{equation}
% \begin{aligned} \label{eq:obj}
%     \mathcal{L} = & \sum_i^n \sum_{y} ( \mathbf{1}(Y_{i}=y) \cdot log(\hat{\mathbf{Y}}[y]) 
%      + \eta \cdot \sum_{j}^{K} \lambda_j \cdot \mathcal{R}(X_{S,j}, \hat{Y})\\& + \beta \|\boldsymbol{\lambda}\|_2^2, \\
%     \text{s.t.} \quad &   \lambda_j > 0, \forall j; \quad \sum_{j}^{K} \lambda_j = 1
% \end{aligned}
% \end{equation}
where $\eta$ sets the weights of regularization term, and $\boldsymbol{\theta}$ is the set of parameters of the classifier. 

Eq.(\ref{eq:obj_1}) can lead to a trivial solution, i.e., to minimize the cost function, it tends to set $\lambda_j$ corresponding to the smallest $\mathcal{R}(\mathbf{x}^j, \hat{\mathbf{y}})$ to $1$ and others to $0$. To alleviate this issue, we add $\|\boldsymbol{\lambda}\|_2^2$ to penalize $\lambda_j$ being close to 1. Thus, The final objective function of {\method} is
\begin{equation}
\begin{aligned} \label{eq:obj}
    \min_{\boldsymbol{\theta}, \boldsymbol{\lambda}} \quad & \mathcal{L}_{cls} + \eta \cdot {\sum}_{j=1}^{K} \lambda_j \cdot \mathcal{R}(\mathbf{x}^{j}, \hat{\mathbf{y}}) + \beta \|\boldsymbol{\lambda}\|_2^2 \\ 
    \text{s.t.} \quad &   \lambda_j \ge 0,  \forall f_j \in \mathcal{F}_S; \quad {\sum}_{j=1}^{K} \lambda_j = 1
\end{aligned}
\end{equation}
where $\beta$ is used to control the contribution of $\|\boldsymbol{\lambda}\|_2^2$.

%To avoid the trivial solution of setting $\lambda_j$ corresponding to the smallest $\mathcal{R}(\mathbf{F}_{S,j}, \hat{Y})$ to $1$ and the others to $0$, we add $\|\lambda\|_2^2$ to penalize $\lambda_j$ being too close to 1, where $\beta$ is used to control the contribution of $\|\boldsymbol{\lambda}\|_2^2$.  % It can be seen that $\boldsymbol{\lambda}$ is updated to reduce the regularization loss, and $\nabla_{\lambda_{j}}\mathcal{L}$ is equal to the value of $\mathcal{R}(\mathbf{F}_{S,j}, \hat{Y})$. % During training, we project $\boldsymbol{\lambda}$ after each update step to make sure the constraints are satisfied. 

\section{Optimization Algorithm} \label{sec:optimization_algorithm}
The objective function in Eq.(\ref{eq:obj}) is constrained optimization, which is difficult to be optimized directly. We take the alternating direction optimization~\cite{goldstein2014fast} strategy to update $\theta$ and $\boldsymbol{\lambda}$ iteratively. The basic idea is to update one variable with the other one fixed at each step, which can ease the optimization process. Next, we give the details.

%\subsection{Updating $\boldsymbol{\theta}$}

\textbf{Update $\boldsymbol{\theta}$.} To optimize $\boldsymbol{\theta}$,  we fix $\boldsymbol{\lambda}$ and remove terms that are irrelevant to $\boldsymbol{\theta}$, which arrives at
\begin{equation}
    \min_{\boldsymbol{\theta}} ~ \mathcal{L}_{cls} + \eta {\sum}_{j=1}^{K} \lambda_j \cdot \mathcal{R}(\mathbf{x}^{j}, \hat{\mathbf{y}})
\end{equation}
%$\mathcal{L}_{cls}+\eta \cdot \mathcal{R}_{related}$. 
This is a non-constrained cost function, and we can directly apply gradient descent to learn $\boldsymbol{\theta}$.

%\vspace*{0.5em}

%\subsection{UPDATE $\boldsymbol{\lambda}$} 
\textbf{UPDATE $\boldsymbol{\lambda}$.}
Then, given $\boldsymbol{\theta}$ at the current step, $\boldsymbol{\lambda}$ can be obtained through solving the following optimization problem:
\begin{equation}
\begin{aligned}
    \boldsymbol{\lambda} =
    & \argmin_{\boldsymbol{\lambda}}  {\sum}_{j=1}^{K} \lambda_j \cdot \mathcal{R}(\mathbf{x}^j, \hat{\mathbf{y}}) + \beta \|\boldsymbol{\lambda}\|_2^2, \\
    \text{s.t.} ~ &   -\lambda_j \leq 0, \forall f_j \in \mathcal{F}_S; \quad {\sum}_{j=1}^{K} \lambda_j - 1 = 0
\end{aligned}
\end{equation}
It is a convex primal problem, and strong duality holds as it follows \textit{Slater's condition}. For simplicity of notation, we use $\mathcal{R}_j$ to represent $\mathcal{R}(\mathbf{x}^{j}, \hat{\mathbf{y}})$. Then, we can solve this problem using Karush-Kuhn-Tucker(KKT)~\cite{mangasarian1994nonlinear} conditions as: 
\begin{equation}
    \begin{cases}
    \mathcal{R}_j + 2\beta \cdot \lambda_j - u_j + v = 0, ~ \forall j;~~~\textit{(stationary)} \\
    u_j \cdot \lambda_j = 0, ~ \forall j; ~~~~\textit{(complementary slackness)}\\
    \lambda_j \geq 0 \quad \forall j; \quad \sum_{j=1}^K\lambda_j = 1; \textit{(primal feasibility)}\\
    u_j \geq 0 \quad \forall j.
    \end{cases}
\end{equation}
In the above equation, $\boldsymbol{u}$ and $v$ are Lagrange multipliers. From the stationary condition, we can get:
\begin{equation}
    \lambda_j = \frac{u_j-v-\mathcal{R}_j}{2\cdot \beta}, \quad  j=1,\dots,K
\end{equation}
Eliminating $\boldsymbol{u}$ using complementary slackness, we have:
\begin{equation}
    \begin{cases}
    \lambda_j = 0, \quad\quad\quad \text{if}~ u_j = v+\mathcal{R}_j \geq 0; \\
    \lambda_j = \frac{-v-\mathcal{R}_j}{2\cdot \beta},  \quad \text{if} ~ u_j = 0; \\
    \lambda_j \geq 0 \quad \forall j; \quad \sum_j^{K} \lambda_j = 1
    \end{cases}
\end{equation}
From this condition, we know that $\lambda_j=\max \{ 0, \frac{-v-\mathcal{R}_j}{2\cdot \beta}\}$. Since $\sum_{j=1}^{K} \lambda_j = 1$, $v$ can be computed via solving the following equation:
\begin{equation}\label{eq:opti_v}
    {\sum}_{j=1}^K \max \{ 0, -v-\mathcal{R}_j\} = 2\beta.
\end{equation}
% \begin{equation}
%     \sum_{j:\mathcal{R}_j \leq -v}(-v - \mathcal{R}_{j}) = 2\beta.
% \end{equation}
Solving the above equation can be done as follows: we first rank $\mathcal{R}_j$ in descending order as $\mathcal{R}_j'$, i.e., $\mathcal{R}_{j-1}' \ge \mathcal{R}_j'$. Assume that $v$ is within $[-\mathcal{R}_{l-1}', -\mathcal{R}_l']$, then the above equation is reduced to
\begin{equation}
    {\sum}_{j=l}^K -v-\mathcal{R}_j' = 2\cdot \beta
\end{equation}
Then, we have
\begin{equation}
    v=-\frac{2\cdot \beta +\sum_{j=l}^{K}\mathcal{R}_j'}{K-l+1}
\end{equation}
If $v=-\frac{2\cdot \beta +\sum_{j=l}^{K}\mathcal{R}_j'}{K-l+1} \in [-\mathcal{R}_{l-1}', -\mathcal{R}_l']$, it is a valid solution; otherwise, it is invalid. We do this for every interval and find $v$.
% \begin{equation} \label{eq:opti_v}
%     \sum_{j}^{l-1}(\mathcal{R}'_{l} - \mathcal{R}'_{j}) \leq 2\beta, \quad
%     \sum_{j}^{l}(\mathcal{R}'_{l+1} - \mathcal{R}'_{j}) > 2\beta.
% \end{equation}
% \begin{aligned}
%     & \sum_{j}^{l-1}(\mathcal{R}'_{l} - \mathcal{R}'_{j}) \leq 2\beta, \\
%     & \sum_{j}^{l}(\mathcal{R}'_{l+1} - \mathcal{R}'_{j}) > 2\beta.
% \end{aligned}
% \end{equation}
% Then, $v$ can be directly computed as:
% \begin{equation}
%     \mathcal{R}'_l + \frac{(2\beta-\sum_{j}^{l-1}(\mathcal{R}'_{l} - \mathcal{R}'_{j}))}{l}
% \end{equation}
With $v$ learned, we can calculate $\boldsymbol{\lambda}$ as:
\begin{equation} \label{eq:opti_lambda}
    \lambda_j=\max \{ 0, \frac{-v-\mathcal{R}_j}{2\cdot \beta}\}
\end{equation}
% \begin{equation}
% \label{eq:opti}
%     \lambda_j = 
%     \begin{cases}
%      \frac{-v-\mathcal{R}_j}{2\cdot \beta}, \quad j < l \\
%     0,      \quad  j \geq l.
%     \end{cases}
% \end{equation}

%\subsection{Training Algorithm}
\textbf{Training Algorithm.} With the updating rules above, the full pipeline of the training algorithm for FairRF can be summarized in Algorithm 1 in the supplementary material.

\section{Experiment} \label{sec:experiment}
In this section, we conduct experiments to evaluate the effectiveness of the proposed {\method} in terms of both fairness and classification performance when sensitive attributes are unavailable. %  Regularization is applied to constrain correlation coefficient between related features and prediction results, in order to make the classifier work fairly across protected groups with unseen sensitive attributes. 
In particular, we aim to answer the following research questions:
\begin{itemize}[leftmargin=*]
    \item \textbf{RQ1} Can the proposed {\method} achieve fairness without sensitive attributes while maintain high accuracy?% Can models trained with regularization on related features work fairly on different sensitive attributes?
    %\item \textbf{RQ2} Can it outperform previous approaches in this unseen sensitive attribute scenario?
    \item \textbf{RQ2} How would {\method} perform when the provided $\mathcal{F}_{S}$ contains misidentified related features or is incomplete? 
    \item \textbf{RQ3} How would different choices of hyper-parameters influence the performance of {\method}?
\end{itemize}
%We begin by presenting the datasets and implementation details, followed by baselines and experiments configuration. We then conduct experiments and ablation studies on real-world datasets to answer these three questions.

\subsection{Datasets}

We conduct experiments on three publicly available benchmark datasets, including Adult~\cite{asuncion2007uci}, COMPAS~\cite{Julia2016machine} and LSAC~\cite{wightman1998lsac}.% following previous literature~\cite{lahoti2020fairness,yan2020fair}.
\begin{itemize}[leftmargin=*]
    \item \textbf{ADULT}\footnote{https://archive.ics.uci.edu/ml/machine-learning-databases/adult/}: It contains $45,221$ records of personal yearly income, with binary label indicating if the yearly salary is over or under $\$50K$. Gender is considered as sensitive attribute. and we select age, relation and marital status as $\mathcal{F}_s$.
    \item \textbf{COMPAS}\footnote{https://github.com/propublica/compas-analysis}: This dataset assesses the possibility of recidivism within a certain future, containing $11,750$ criminal records collected in US.The race of each defendant is the sensitive attribute. In constructing $\mathcal{F}_s$, score, decile text and sex are selected. 
    \item \textbf{LSAC}\footnote{http://www.seaphe.org/databases.php}: It contains $65,307$ admissions data from $25$ law schools in US over the 2005, 2006, and 2007 admission cycles. Labels indicate whether each candidate successfully pass the bar exam or not, and their gender information is considered as sensitive. For this dataset, we use race, year and residence as $\mathcal{F}_s$.
\end{itemize}

We make the train:eval:test splits as $5:2:3$. \textit{Note that for all three datasets, features in $\mathcal{F}_s$ are selected following existing analysis or prior domain knowledge.} For example, in COMPAS, biases towards race have been found to exist in score and decile text~\cite{Julia2016machine}. The correlation between race and gender is also from reports by U.S. Bureau of Justice Statistics(BJS). Since race is the sensitive attribute of the dataset, we include score, decile text and gender in $\mathcal{F}_S$. 

\subsection{Experimental Settings}
\subsubsection{Baselines}
To evaluate the effectiveness of {\method}, we first compare it with the vanilla model and sensitive-attribute-aware model, which can be treated as the lower and upper bound of our model's performance:
\begin{itemize}[leftmargin=*]
    \item \textbf{Vanilla model}: It directly uses the base classifier without any regularization terms. It is used to show the performance without fairness-assuring algorithm taken. 
    \item \textbf{ConstrainS}: In this baseline, we assume that the sensitive attribute of each data sample is known. We add the correlation regularization between sensitive attribute vector $\mathbf{s}$ and model output $\hat{\mathbf{y}}$, i.e., $\mathcal{R}(\mathbf{s},\hat{\mathbf{y}})$. It sets a reference point for the performance of the proposed framework. Note that for all the other baselines and our model, $\mathbf{s}$ is unknown.
\end{itemize}
We also include following representative approaches in fair learning without sensitive attributes as baselines:
\begin{itemize}[leftmargin=*]
    \item \textbf{KSMOTE}~\cite{yan2020fair}: It performs clustering to obtain pseudo groups, and use them as substitute. The model is regularized to be fair with respect to those pseudo groups. 
    \item \textbf{RemoveR}: This method directly removes all candidate related features, i.e., $\mathcal{F}_S$. We design this baseline in order to validate the benefits of our proposed method in regularizing related features. 
    \item \textbf{ARL}~\cite{lahoti2020fairness} It follows Rawlsian principle of Max-Min welfare for distributive justice. It optimizes model's performance through re-weighting regions detected by an adversarial model. 
\end{itemize}

\textit{Note that the fairness formulation of ARL is different from the group fairness we focus on.} ARL~\cite{lahoti2020fairness} is inefficient in obtaining demographic fairness by design, which is also verified by our experiments. Although not working on the same fairness definition, we still include it as one baseline for completeness of the experiment.
% \begin{itemize}
%     \item \textbf{ARL}~\cite{lahoti2020fairness} Adversarial Re-weight optimizes model's performance through re-weighting under-represented regions detected by an adversarial model.
% \end{itemize}
%\textit{Note that} 

\subsubsection{Configurations}
For KSMOTE, we directly use the code provided by ~\cite{yan2020fair}. For all other approaches, we implement a multi-layer perceptron (MLP) network with three layers as the backbone classifier. The two hidden dimensions are $64$ and $32$. Adam optimizer is adopted to train the model, with initial learning rate as $0.001$.

\begin{table}[t]
  \setlength{\tabcolsep}{4.5pt}
  %\normalsize
  % \small
  \centering
  \caption{Comparison of different approaches on ADULT.} 
  \label{tab:result_adult}
  \vskip -1em
  \begin{tabular}{p{1.5cm} | p{1.8cm} p{1.8cm}  p{1.8cm}  }
    \hline 
    % \multirow{2}{4em}{Methods}
    Methods
    & ACC  & $\Delta_{EO}$ & $\Delta_{DP}$   \\
    % & (Avg.) & & &(Avg.) & & &(Avg.) & & \\
    % \hline
    \hline
    %Origin & & & & & & & & &  \\
    %\hline
    Vanilla & $0.856\pm0.001$ & $0.046\pm0.006$ & $0.089\pm0.005$  \\
    ConstrainS & $0.845\pm0.002$ & $0.040\pm0.004$ & $0.058\pm0.003$  \\
    \hline
    ARL  & $0.861\pm0.003$ & $0.034\pm0.012$ & $0.141\pm0.008$  \\
    \hline
    KSMOTE  & $0.560\pm0.002$ & $0.141\pm0.031$ & $0.120\pm0.022$ \\
    RemoveR & $0.801\pm0.010$ & $0.124\pm0.004$ & $0.071\pm0.002$  \\
    \hline
    %uniform & $0.8293$ & $0.1064$ & $0.0668$ & $0.6403$ & $0.1572$ & $0.1135$ & $0.7791$ & $0.0229$ & $0.0070$ \\
    {\method} & $0.832\pm0.001$ & $\textbf{0.025}\pm0.009$ & $\textbf{0.066}\pm0.004$  \\
    \hline
  \end{tabular}
  \vskip 0.5em
  
  \caption{Comparison of different approaches on COMPAS} 
  \label{tab:result_compas}
  \vskip -1em
  \begin{tabular}{p{1.5cm} | p{1.8cm}  p{1.8cm}  p{1.8cm} }
    \hline 
    % \multirow{2}{4em}{Methods}
    Methods
     & ACC  & $\Delta_{EO}$ & $\Delta_{DP}$ \\
    % & (Avg.) & & &(Avg.) & & &(Avg.) & & \\
    % \hline
    \hline
    %Origin & & & & & & & & &  \\
    %\hline
    Vanilla & $0.681\pm0.004$ & $0.242\pm0.021$ & $0.171\pm0.015$  \\
    ConstrainS & $0.674\pm0.002$ & $0.154\pm0.032$ & $0.122\pm0.031$  \\
    
    \hline
    ARL  & $0.672\pm0.023$ &  $0.197\pm0.042$ & $0.286\pm0.033$  \\
    \hline
    KSMOTE  & $0.601\pm0.021$ &  $0.203\pm0.042$ & $0.151\pm0.023$   \\
    RemoveR  & $0.595\pm0.024$ & $0.205\pm0.049$ & $0.185\pm0.024$  \\
    \hline
    %uniform & $0.8293$ & $0.1064$ & $0.0668$ & $0.6403$ & $0.1572$ & $0.1135$ & $0.7791$ & $0.0229$ & $0.0070$ \\
    {\method} & $0.661\pm0.009$ & $\textbf{0.166}\pm0.022$ & $\textbf{0.143}\pm0.021$  \\
    \hline
  \end{tabular}
  \vskip 0.5em

  \caption{Comparison of different approaches on LSAC.}
  \label{tab:result_law}
  \vskip -1em
  \begin{tabular}{p{1.5cm}  | p{1.8cm}  p{1.8cm}  p{1.8cm} }

    \hline 
    % \multirow{2}{4em}{Methods}
    Methods
    & ACC  & $\Delta_{EO}$ & $\Delta_{DP}$ \\
    % & (Avg.) & & &(Avg.) & & &(Avg.) & & \\
    % \hline
    \hline
    %Origin & & & & & & & & &  \\
    %\hline
    Vanilla  &$0.805\pm0.001$& $0.042\pm0.007$ & $0.016\pm0.004$ \\
    ConstrainS & $0.801\pm0.001$ &$0.014\pm0.007$ & $0.004\pm0.002$ \\
    
    \hline
    ARL  & $0.811\pm0.005$  & $0.029\pm0.029$ & $0.022\pm0.013$ \\
    \hline
    KSMOTE  & $0.722\pm0.012$  & $0.028\pm0.062$ & $0.012\pm0.041$ \\
    RemoveR & $0.763\pm0.002$ & $0.037\pm0.024$ & $0.015\pm0.006$ \\
    \hline
    %uniform & $0.8293$ & $0.1064$ & $0.0668$ & $0.6403$ & $0.1572$ & $0.1135$ & $0.7791$ & $0.0229$ & $0.0070$ \\
    {\method} & $0.796\pm0.002$ & $\textbf{0.023}\pm0.008$ & $\textbf{0.007}\pm0.004$ \\
    \hline
  \end{tabular}
  \vskip -1em
\end{table}

\subsubsection{Evaluation Metrics}
To measure the fairness, following existing work on fair models~\cite{verma2018fairness,yan2020fair}, we adopt two widely used evaluation metrics, i.e., equal opportunity and demographic parity, which are defined as follows:

\vspace*{0.5em}
\noindent{}\textbf{Equal Opportunity}~\cite{mehrabi2019survey} Equal opportunity requires that the probability of positive instances with arbitrary protected attributes $i, j$ being assigned to a positive outcome are equal:
\begin{equation}
    {\mathbb{E}}(\hat{y} \mid S=i, y=1) = {\mathbb{E}}(\hat{y} \mid S=j, y=1),
\end{equation}
where $\hat{y}$ is the output of model $g_{\theta}$, representing the probability of being predicted as positive. In the experiments, we report difference in equal opportunity($\Delta_{EO}$):
\begin{equation}
    \Delta_{EO} =  |{\mathbb{E}}(\hat{y} \mid S=i, y=1) - {\mathbb{E}}(\hat{y} \mid S=j, y=1)|
\end{equation}
%\vspace*{0.2em}

\noindent{}\textbf{Demographic Parity}~\cite{mehrabi2019survey} Demographic parity requires the behavior of prediction model to be fair on different sensitive groups. Concretely, it requires that the positive rate across sensitive attributes are equal:
\begin{equation}
    {\mathbb{E}}(\hat{y} \mid S=i) = {\mathbb{E}}(\hat{y} \mid S=j), \forall i, j
\end{equation}
Similarly, in the experiment, we report the difference in demographic parity($\Delta_{DP}$):
\begin{equation}
    \Delta_{DP} = |{\mathbb{E}}(\hat{y} \mid S=i) - {\mathbb{E}}(\hat{y} \mid S=j)|
\end{equation}
%\vspace*{0.2em}

% In experiments, we report difference in equal opportunity($\Delta_{EO}$): $|{\mathbb{E}}(\hat{y} \mid S=i, y=1) - {\mathbb{E}}(\hat{y} \mid S=j, y=1)|$ and difference in demographic parity($\Delta_{DP}$): $|P(\hat{y} \mid S=i) - P(\hat{y} \mid S=j)|$. 
Equal opportunity and demographic parity measure the fairness from different perspectives. Equal opportunity requires similar performance across protected groups, while demographic parity is more focused on fair demographics. The smaller $\Delta_{EO}$ and $\Delta_{DP}$ are, the more fair a model is. Furthermore, to measure the classification performance, \textbf{accuracy} (ACC) is also reported.

\subsection{Classification Performance Comparison}\label{sec:mainEXP}
To answer \textbf{RQ1}, we fix the base classifier as MLP and conduct classification on all three datasets. For all the baselines, the hyperparameters are tuned via grid search on the validation dataset. In particular, for {\method}, $\beta$ is set to $0.5$ on ADULT, $0.8$ on COMPAS, and $1.0$ on LSAC. $\eta$ is set as $0.15$ for COMPAS and $0.3$ for other two datasets. More details on the hyperparameters sensitivity will be discussed in Sec~\ref{sec:sensitive}. Each experiment is conducted $5$ times and the average performance in terms of accuracy, $\Delta_{EO}$ and $\Delta_{DP}$ with standard deviation are reported in Table~\ref{tab:result_adult}, Table~\ref{tab:result_compas} and Table \ref{tab:result_law}.  From the tables, we make the following observations:
\begin{itemize}[leftmargin=*]
    \item Constraining related features can help the model to perform fairer on sensitive groups. For example, compared with vanilla approach in which no fair-learning techniques are applied, {\method} shows a clear improvement w.r.t Equal Opportunity and Demographic Parity across all three datasets;
    \item {\method} improves the fairness without causing significant performance drop, and works stably. No pre-computed clusters are required, and it does not involve training an adversarial model, hence {\method} can get results with less deviation compared to ARL and KSMOTE; 
    \item Compared with baselines without sensitive attribute, {\method} is effective for both two fairness metrics; while other approaches such as ARL is able to improve on ``equal opportunity'', but the performance would drop w.r.t ``demographic parity''. This is because {\method} is able to learn $\lambda_j$ to balance the fairness and accuracy.
\end{itemize}

\begin{figure*}[t]
  \centering
    \captionsetup[subfigure]{aboveskip=-2pt,belowskip=-2pt}
    \subfloat[ACC]{
		\includegraphics[width=0.27\textwidth]{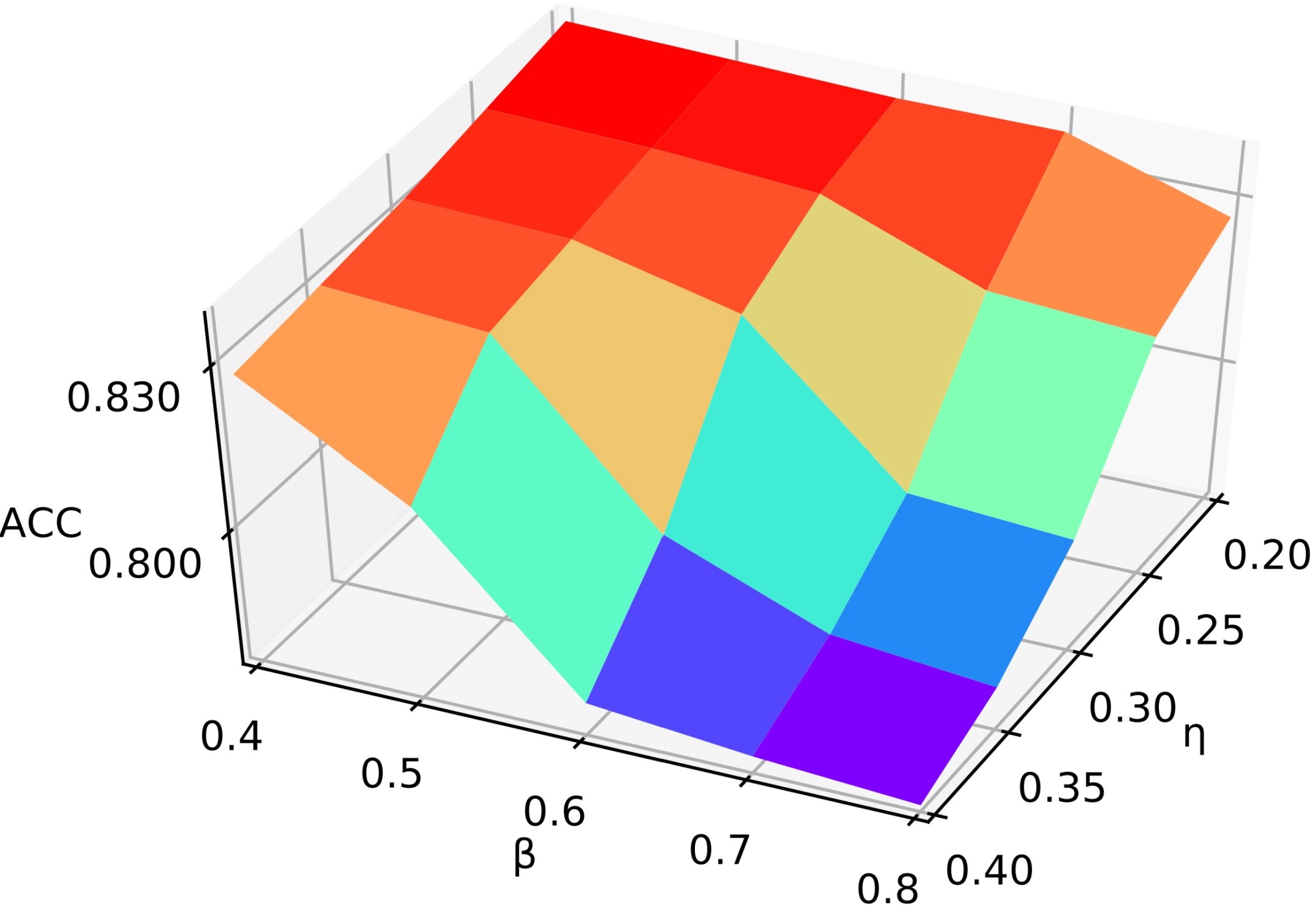}}~~
    \subfloat[$\Delta_{EO}$]{
		\includegraphics[width=0.27\textwidth]{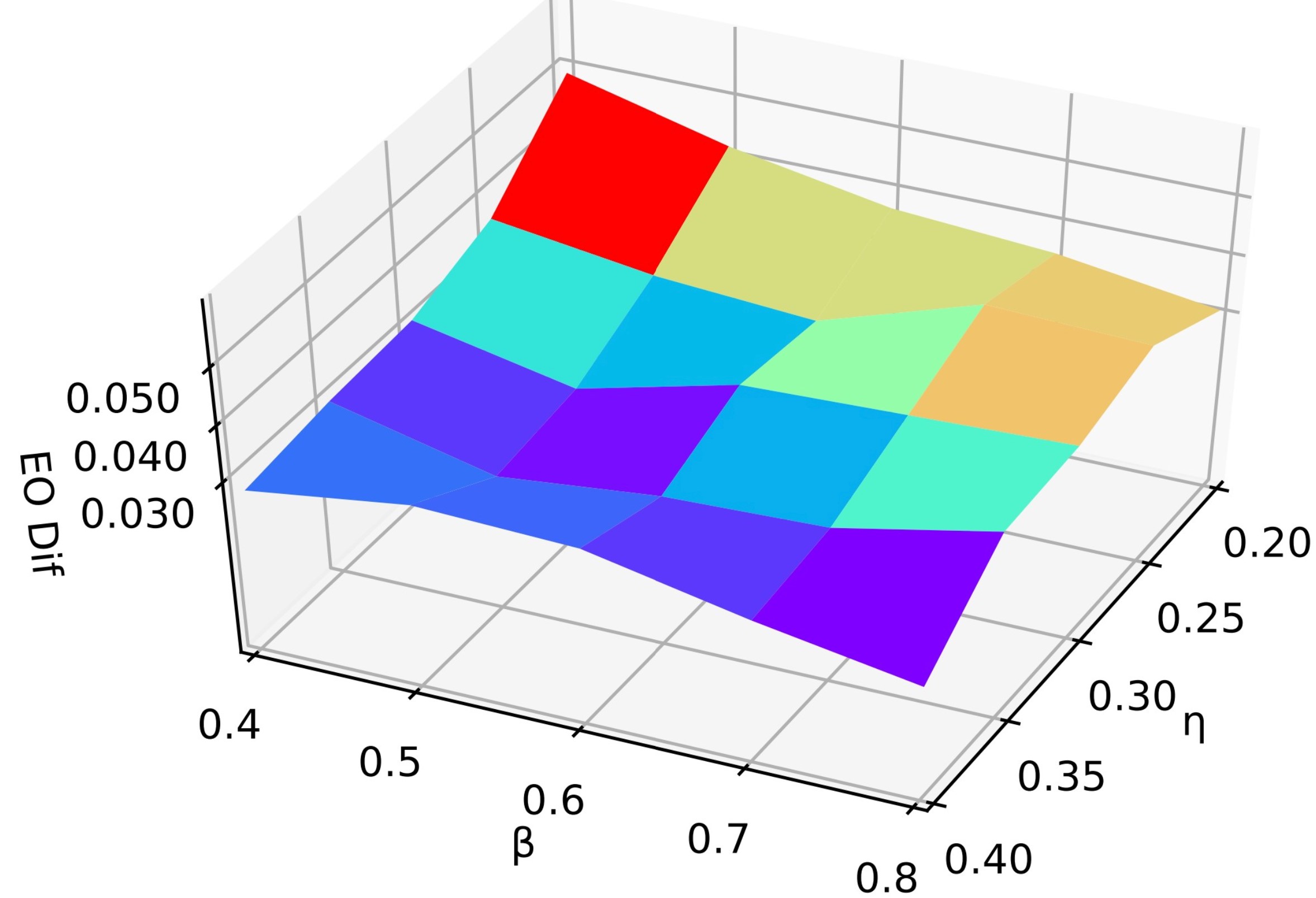}}~~
    \subfloat[$\Delta_{DP}$]{
		\includegraphics[width=0.27\textwidth]{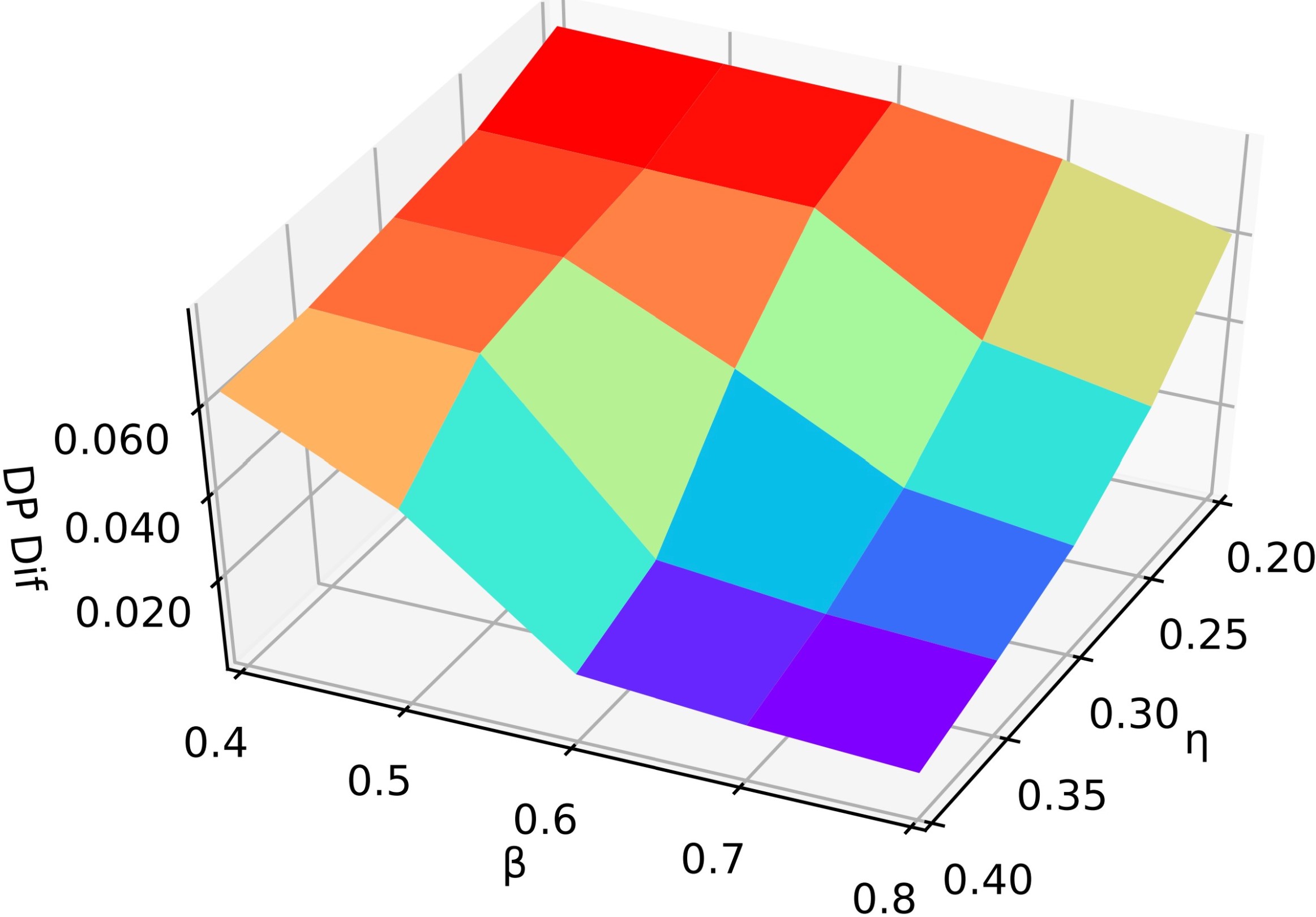}}
	\vskip -1.5em
    \caption{Parameter Sensitivity on ADULT}\label{fig:parameter_sensitivity}
    \vskip -1em
\end{figure*}

\subsection{Impact of the Quality of $\mathcal{F}_S$ on {\method}}
In this section, we conduct experiment to investigate the impact of the quality of $\mathcal{F}_S$ on the performance of {\method} to answer \textbf{RQ2}. In particular, we consider the following variants of {\method}:
\begin{itemize}[leftmargin=*]
    \item \textbf{Random}: We randomly select five sets of $\mathcal{F}_S$ with the same number of attributes as {\method}. Average results are reported. We use it to show the influence of prior knowledge.
    \item \textbf{Fix-$\boldsymbol{\lambda}$}: The same $\lambda_i$ is adopted for all related features, and its value is not automatically updated during training. Selected related features are exactly the same as those chosen in {\method}.
    \item \textbf{Top-1}: It uses only the most-effective related features. We test all candidates and select the one that achieves highest performance when used as related feature, and report its performance.
    \item \textbf{ConstrainAll}: It includes all features in $\mathcal{F}_S$, i.e., all features are treated as related features. This is used to show if noisy features are included or no prior knowledge about related feature is given,  {\method} can still work. We also learn $\boldsymbol\lambda$ for this variant.%All other settings are the same as {\method}. 
    \item \textbf{Noisy}: Its contains features randomly sampled from both $\mathcal{F}_{S}$ and non-related attributes. In implementation, we randomly replace one attribute in $\mathcal{F}_{S}$ with non-related ones.
\end{itemize}
For all these baselines, hyper-parameters are found via grid search, and experiments are conducted for $5$ times randomly. From Table \ref{tab:ablation} and \ref{tab:ablation_compas}, we can make following observations:
\begin{itemize}[leftmargin=*]
    \item {\method} can still bring improvements when $\mathcal{F}_{S}$ is inaccurate. The variant Noisy is shown to be effective across ADULT and COMPAS datasets.
    \item In the extreme case that no prior knowledge is available, {\method} still has potentials on fairness metrics compared with vanilla model, as shown by Random and ConstrainAll. It again shows that {\method} can cope with little domain knowledge scenario. 
    \item {\method} benefits from automatically learning the importance of each given related attribute. Compared with Fix-$\boldsymbol\lambda$, {\method} shows a much stronger fairness in terms of equal opportunity, and achieves better accuracy at the same time. 
    \item {\method} shows a moderate improvement compared with Top-1. However, Top-1 requires careful selection of the most effective related feature, while {\method} can achieve better performance with less prior domain knowledge; 
\end{itemize}
Due to space limitation, we only report the results on ADULT and COMPAS, but similar observations can be made on LSAC.

\begin{table}[t]
  \setlength{\tabcolsep}{4.5pt}
  %\normalsize
  \centering
  %\small
  \caption{Comparison of different strategies in selecting related features on ADULT. }
  \label{tab:ablation}
  \vskip -1em
  \begin{tabular}{| p{2.3cm} | p{1.6cm} p{1.6cm}  p{1.6cm}|  }
    \hline
    % &  \multicolumn{3}{|c|}{ADULT} \\ \hline
    % \multirow{2}{4em}{Methods} & ACC  & $\Delta_{EO}$ & $\Delta_{DP}$   \\
    % & (Avg.) & & \\
    Methods &   ACC &   $\Delta_{EO}$  &   $\Delta_{DP}$ \\
    \hline
    \hline
    %Origin & & & & & & & & &  \\
    %\hline
    %top-1 & $0.8319$ & $0.0420$ & $0.0782$ & $0.6529$ & $0.2024$ & $0.1258$ & $0.7983$ &$0.0212$ & $0.0077$ \\
    %uniform & $0.8293$ & $0.1064$ & $0.0668$ & $0.6403$ & $0.1572$ & $0.1135$ & $0.7791$ & $0.0229$ & $0.0070$ \\
    %multiple & $0.8325$ & $0.0407$ & $0.0728$ & $0.6536$ & $0.1716$ & $0.1228$ & $0.7898$ & $0.0209$ & $0.0069$ \\
    %{\method} & $0.8407$ & $0.0391$ & $0.0762$ & $0.6458$ & $0.1639$ & $0.1172$ & $0.7942$ & $0.0261$ & $0.0093$ \\
    Vanilla & $0.856\pm0.001$ & $0.046\pm0.006$ & $0.089\pm0.005$ \\
    \hline
    Random & $0.830\pm0.001$  & $0.041\pm0.012$  & $0.057\pm0.007$  \\
    Top-1 & $0.830\pm0.002$ & $0.029\pm0.008$ & $0.067\pm0.002$ \\
    ConstrainAll & $0.835\pm0.001$ & $0.035\pm0.005$ & $0.068\pm0.003$ \\
    Noisy & $0.834\pm0.002$ & $0.030\pm0.011$ & $0.068\pm0.006$ \\
    \hline
    Fix-$\boldsymbol\lambda$ & $0.822\pm0.002$ & $0.065\pm0.007$ & $0.057\pm0.004$  \\
    {\method} & $0.832\pm0.001$ & $\textbf{0.025}\pm0.009$ & $0.066\pm0.004$ \\
    
    \hline
  \end{tabular}
  %\vskip -1em
  
    \caption{Comparison of different strategies in selecting related features on COMPAS. }
  \label{tab:ablation_compas}
  \vskip -1em
  \begin{tabular}{| p{2.3cm} | p{1.6cm} p{1.6cm}  p{1.6cm}|  }
    \hline
    Methods &   ACC &   $\Delta_{EO}$  &   $\Delta_{DP}$ \\
    \hline
    \hline
    Vanilla & $0.681\pm0.004$ & $0.242\pm0.021$ & $0.171\pm0.015$ \\
    \hline
    Random & $0.637\pm0.006$  & $0.226\pm0.028$  & $0.161\pm0.016$  \\
    Top-1 & $0.648\pm0.007$ & $0.183\pm0.016$ & $0.164\pm0.013$ \\
    ConstrainAll & $0.651\pm0.004$ & $0.235\pm0.012$ & $0.168\pm0.008$ \\
    Noisy & $0.653\pm0.006$ & $0.219\pm0.023$ & $0.154\pm0.019$ \\
    \hline
    Fix-$\boldsymbol\lambda$ & $0.631\pm0.011$ & $0.256\pm0.025$ & $0.159\pm0.018$  \\
    {\method} & $0.661\pm0.009$ & $\textbf{0.166}\pm0.022$ & $\textbf{0.143}\pm0.021$ \\
    
    \hline
  \end{tabular}
  \vskip -1.5em
\end{table}

% \begin{table}[t]
%   \setlength{\tabcolsep}{4.5pt}
%   %\normalsize
%   \centering
%   %\small
%   \caption{Comparison of different strategies in selecting related features on COMPAS. }
%   \label{tab:ablation_compas}
%   \vskip -1em
%   \begin{tabular}{| p{2.3cm} | p{1.6cm} p{1.6cm}  p{1.6cm}|  }
%     \hline
%     Methods &   ACC &   $\Delta_{EO}$  &   $\Delta_{DP}$ \\
%     \hline
%     \hline
%     Vanilla & $0.681\pm0.004$ & $0.242\pm0.021$ & $0.171\pm0.015$ \\
%     \hline
%     Top-1 & $0.648\pm0.007$ & $0.183\pm0.016$ & $0.164\pm0.013$ \\
%     Random & $0.637\pm0.006$  & $0.226\pm0.028$  & $0.161\pm0.016$  \\
%     Fix-ConstrainAll & $0.663\pm0.004$ & $0.219\pm0.019$ & $0.203\pm0.015$ \\
%     ConstrainAll & $0.651\pm0.004$ & $0.235\pm0.012$ & $0.168\pm0.008$ \\
%     \hline
%     Fix-$\boldsymbol\lambda$ & $0.631\pm0.011$ & $0.256\pm0.025$ & $0.159\pm0.018$  \\
%     {\method} & $0.639\pm0.009$ & $0.186\pm0.022$ & $0.141\pm0.021$ \\
    
%     \hline
%   \end{tabular}
%   %\vskip -1em
% \end{table}

\subsection{Parameter Sensitivity Analysis}\label{sec:sensitive}
In this subsection, we analyze the sensitivity of {\method} on hyperparameters $\eta$ and $\beta$. $\eta$ controls the importance of coefficient regularization term, and $\beta$ can adjust the distribution of learned $\boldsymbol\lambda$. We vary $\eta$ as $\{0.2,0.25,0.3,0.35,0.4\}$ and $\beta$ as $\{0.4,0.5, 0.6,0.7,0.8\}$. Other settings are the same as {\method}. This experiment is performed on ADULT, with results shown in Figure~\ref{fig:parameter_sensitivity}. From the figure, we can observe that: (i) Larger $\eta$ will achieve fairer predictions, but may also cause severe drop in accuracy when it is larger than some thresholds; (ii) Generally, smaller $\beta$ requires larger $\eta$ to achieve fairness. Small $\beta$ allows learned $\boldsymbol\lambda$ to be sparse. As a result, a large portion of coefficient regularization term could be enforced on less-discriminative attributes that are less-related at the same time; and (iii) $\beta$ encourages learned $\boldsymbol\lambda$ to be uniform, resulting a faster drop in accuracy when $\eta$ goes large. These observations could help to find suitable hyper-parameter choices in other applications.
% \begin{itemize}[leftmargin=*]
%     \item Larger $\eta$ will achieve fairer predictions, but may also cause severe drop in accuracy when it is larger than some thresholds;
%     \item Generally, smaller $\beta$ requires larger $\eta$ to achieve fairness. Small $\beta$ allows learned $\boldsymbol\lambda$ to be sparse. As a result, a large portion of coefficient regularization term could be enforced on less-discriminative attributes that are less-related at the same time;
%     \item Larger $\beta$ encourages learned $\boldsymbol\lambda$ to be uniform, resulting a faster drop in accuracy when $\eta$ goes large.
% \end{itemize}
% These observations could help to find suitable hyper-parameter choices in other applications.

\subsection{Case Study on $\boldsymbol \lambda$}
In this subsection, we conduct case studies to analyze the behavior of {\method} in  learning $\boldsymbol\lambda$, i.e., the weights of related attributes. Specifically, we calculate the ground-truth correlation between the sensitive attribute $S$ and others are computed, and a set of attributes with varying range of correlation coefficient magnitudes are selected as $\mathcal{F}_{S}$. $\eta$ and $\beta$ are set using grid search to make sure that fairness is obtained without significant drop in accuracy.We report the distribution of learned $\boldsymbol\lambda$. Results on ADULT and COMPAS are shown in Table~\ref{tab:case_study}.
\begin{table}[t]
  \setlength{\tabcolsep}{4.5pt}
  %\normalsize
  \centering
  %\small
  \caption{Examples of learned $\boldsymbol\lambda$ on a set of selected related attributes. $\rho_{\cdot, Y}$ represents its correlation with class label, and $\rho_{\cdot, S}$ is the correlation with sensitive attributes $S$.}
  \label{tab:case_study}
  \vskip -1em
  \begin{tabular}{ | c c c c | c c c c |  }
    \hline
     \multicolumn{4}{|c}{ADULT}  &   \multicolumn{4}{|c|}{COMPAS} \\
    \hline
     Attr & $\rho_{\cdot, Y}$ & $\rho_{\cdot, S}$ & $\lambda$ & Attr &  $\rho_{\cdot, Y}$ & $\rho_{\cdot, S}$ & $\lambda$  \\
    \hline
    Age & $0.09$ & $0.05$ & $0.51$ & Sex & $0.11$ & $0.07$ & $0.27$\\
    Workclass & $0.11$ & $0.14$ & $0.49$ & Score & $0.31$ & $0.27$ & $0.00$ \\
    Relation & $0.41$ & $0.58$ & $0.00$ & Decile & $0.25$ & $0.24$ & $0.21$ \\
    Education & $0.18$ & $0.06$ & $0.00$ & Duration & $0.02$ & $0.30$ & $0.52$\\
    
    \hline
  \end{tabular}
  \vskip -1em
\end{table}
From the result, we can observe 
\begin{itemize}[leftmargin=*]
    \item {\method} tends to assign higher weight to features that have high correlation with $S$ but small correlation with $Y$. For example, the correlation of ``Duration'' with label is 0.02 and with $S$ is 0.30, {\method} assigns 0.52 to the feature. This is because such features have little effect on model accuracy but introduce a lot of bias.  Assigning a large weight can help achieve fairness with marginal affects on performance;
    \item On the contrary, when a feature $f_j$ has high correlation with $Y$, {\method} tends to assign smaller number to $\lambda_j$ even if the correlation of the feature with $S$ is large. For example, {\method} assigns 0 to ``Relation''. This is because when a feature has high correlation with label, it is important for model prediction. A large weight on fairness regularizer will significantly reduce the accuracy.
\end{itemize}

These observations further demonstrate that by learning $\boldsymbol \lambda$, {\method} can balance the accuracy and fairness.
%that a higher weight tends to be assigned to those correlates with $S$ but not has a high correlation with $Y$ at the same time. For example, although 'Relation' in ADULT and 'Score' in COMPAS are correlated with $S$, they also show a strong correlation with $Y$, which results in a low $\lambda$. While attributes like 'Workclass' and 'Decile', they show a weak correlation with $Y$ while a relatively high correlation with $S$, and are preferred by {\method}.

\subsection{Flexibility of {\method} to Various Backbones}
In the above experiment, we fix the base classifier as MLP. In this section, we investigate if {\method} can also benefit various classifiers to achieve fairness while maintaining high accuracy when the sensitive attributes are unknown. Specifically, we also adopt two other widely-used classifiers as the base classifiers of FairRF, i.e., Linear Regression (LR) and Support Vector Machine (SVM). \textit{The details of experimental setting and results are given in Supplementary Material}. For both models, we find that FairRF only scarifies a little bit of accuracy while significantly improves the fairness. For example, by adding {\method} to LR , $\Delta_{EO}$ drops by 58.5\% while the accuracy only drops by 2\%.

\section{Conclusion} \label{sec:conclusion}
In this paper, we study a novel and challenging problem of exploring related features for learning fair and accurate classifiers without knowing the sensitive attribute of each data sample. We propose a new framework {\method} which utilizes the related features as pseudo sensitive attribute to regularize the model prediction. Our theoretical analysis shows that if the related features are highly correlated with the sensitive attribute, by minimizing the correlation between the related features and model's prediction, we can learn a fair classifier with respect to the sensitive attribute. Since we lack the prior knowledge of the importance of each related feature, we design a mechanism for the model to automatically learn the importance weight of each feature to trade-off their contribution on classification accuracy and fairness.
%In this work, we study a challenging problem: achieving fairness with unseen sensitive attributes. This problem widely exists in real-world applications, but is less-studied previously. 
% We propose to utilize prior domain knowledge, finding attributes related to the real sensitive ones and use them in substitute. We theoretically prove the soundness of this approach, and design an algorithm {\method} to impose fairness constraints. {\method} can learn to dynamically adjust the weight of each related features at the same time, reducing the effort of applying it. 
Experiments on real-world datasets show that the proposed approach is able to achieve more fair performance compared to existing approaches while maintain high classification accuracy when no sensitive attributes are known.

\section{Acknowledgement}
This material is based upon work supported by, or in part by, the National Science Foundation under grants number IIS-1909702 and IIS-1955851. The findings and conclusions in this paper do not necessarily reflect the view of the funding agency.

\bibliographystyle{ACM-Reference-Format}
\balance
\bibliography{acmart}

\newpage 
\appendix

\section{Training Algorithm}

With the updating rules introduced in Section~\ref{sec:optimization_algorithm}, the full pipeline of the training algorithm for {\method} can be summarized in Algorithm~\ref{alg:Framwork}. Before adding the regularization, we first pre-train the model to converge at a good start point in line $3$ in order to prevent correlation constraint from providing noisy signals. Then, from line $5$ to line $13$, we fine-tune the model to be fair w.r.t related features. If not refining related weights, $\boldsymbol{\lambda}$ will stay fixed. Otherwise, it will be updated iteratively with parameter $\theta$, as shown from line $9$ to $12$.

\begin{algorithm}[t]
  \caption{Training Algorithm of {\method}}
  \label{alg:Framwork}
  %\scalebox{0.75}{
  \begin{algorithmic}[1] 
  \REQUIRE %??????????Input
    $\mathbf{X} \in \mathbb{R}^{n \times m}, \mathbf{y} \in \mathbb{R}^{n \times 1}, \mathcal{F}_S$
  \ENSURE $\text{classifier parameters } \boldsymbol{\theta}$
    \STATE Randomly initialize $\boldsymbol{\theta}$; Initialize all entries in $\boldsymbol\lambda$ as $\frac{1}{K}$;
    \FOR{batch in $(X, Y)$}
    \STATE Update $\boldsymbol{\theta}$ based on classification loss of current batch;
    \ENDFOR
    
    \WHILE {Not Converged}
    \FOR{step in MODEL\_TRAIN\_STEP}
    \STATE Update $\boldsymbol{\theta}$ based on Equation~\ref{eq:obj_1};
    \ENDFOR
    \IF{Require learning weight}
    \STATE Obtain $\mathcal{R}_{j}$ for each related feature $\mathbf{x}^j$;
    \STATE Calculate $v$ and  $\boldsymbol{\lambda}$ based on Eq.(\ref{eq:opti_v}) to Eq.(\ref{eq:opti_lambda});
    \ENDIF 
    \ENDWHILE
    
    \RETURN Trained classifier $\theta$.
  \end{algorithmic}
  %}%% resizebox
\end{algorithm}

\section{Implementation on Different Base Model}

In the experiments of main paper, we fix the base classifier as MLP. In this section, we present the incorporation of {\method} into various machine learning models to achieve fairness while maintain high accuracy when the sensitive attributes are unknown. Specifically, in addition to MLP, we also adopt two other widely-used classifiers as the base classifiers of {\method}, i.e., Linear Regression (LR) and Support Vector Machine (SVM). We implement both of them in a gradient-based manner. so that parameters can be optimized alternatively with the regularization term on related features, as in Algorithm~\ref{alg:Framwork}.

\begin{table}[h]
  \setlength{\tabcolsep}{4.5pt}
  %\normalsize
  %\small
  \caption{Effectiveness of {\method} with various base classifiers on ADULT dataset. }
  \label{tab:baseArc}
  \vskip -1em
  \begin{tabular}{| p{1.8cm} | p{1.6cm} p{1.6cm}  p{1.6cm} | }
    \hline
    % &  & \multicolumn{3}{|c|}{Adult}  \\ \hline
    % \multirow{2}{4em}{Models} & Method & ACC  & $\Delta_{EO}$ & $\Delta_{DP}$  \\
    % & & (Avg.) & & \\
     Method &    ACC & $\Delta_{EO}$    &   $\Delta_{DP}$ \\ \hline
    \hline
    %Origin & & & & & & & & &  \\
    %\hline
    LR & $0.832\pm0.004$ & $0.053\pm0.003$ & $0.125\pm0.005$  \\
    {\method}(LR) & $0.815\pm0.008$ & $0.022\pm0.009$ & $0.072\pm0.014$ \\ \hline
    SVM & $0.775\pm0.013$ & $0.083\pm0.008$ & $0.117\pm0.013$  \\
    {\method}(SVM) & $0.775\pm0.015$ & $0.031\pm0.017$ & $0.056\pm0.024$ \\ \hline
    %multiple & $0.8325$ & $0.0407$ & $0.0728$ & $0.6536$ & $0.1716$ & $0.1228$ & $0.7898$ & $0.0209$ & $0.0069$ \\
    MLP & $0.856\pm0.001$ & $0.046\pm0.006$ & $0.089\pm0.005$ \\
    {\method}(MLP) & $0.832\pm0.001$ & $0.025\pm0.009$ & $0.066\pm0.004$  \\
    \hline
  \end{tabular}
  
    \caption{Evaluate effectiveness of {\method} on different base classifiers on COMPAS. }
  \label{tab:baseArc_compas}
  \vskip -1em
  \begin{tabular}{| p{1.8cm} | p{1.6cm} p{1.6cm}  p{1.6cm} | }
    \hline
    % &  & \multicolumn{3}{|c|}{Adult}  \\ \hline
    % \multirow{2}{4em}{Models} & Method & ACC  & $\Delta_{EO}$ & $\Delta_{DP}$  \\
    % & & (Avg.) & & \\
     Method &    ACC & $\Delta_{EO}$    &   $\Delta_{DP}$ \\ \hline
    \hline
    %Origin & & & & & & & & &  \\
    %\hline
    LR & $0.678\pm0.002$ & $0.215\pm0.033$ & $0.198\pm0.026$  \\
    {\method}(LR) & $0.671\pm0.001$ & $0.201\pm0.011$ & $0.146\pm0.008$ \\ \hline
    SVM & $0.664\pm0.013$ & $0.241\pm0.006$ & $0.151\pm0.008$  \\
    {\method}(SVM) & $0.661\pm0.008$ & $0.162\pm0.008$ & $0.134\pm0.013$ \\ \hline
    %multiple & $0.8325$ & $0.0407$ & $0.0728$ & $0.6536$ & $0.1716$ & $0.1228$ & $0.7898$ & $0.0209$ & $0.0069$ \\
    MLP & $0.681\pm0.004$ & $0.242\pm0.021$ & $0.171\pm0.015$ \\
    {\method}(MLP) & $0.661\pm0.009$ & $0.166\pm0.022$ & $0.143\pm0.021$  \\
    \hline
  \end{tabular}
  \vskip -1em
\end{table}

Concretely, we tune the hyperparameters on the validation set. $\eta$ is fixed to $0.4$, and $\beta$ is set to $0.4$ and $0.6$ for LR and SVM, respectively. Adam optimizer is adopted to train them, with the initial learning rate as $0.001$. Each experiment is conducted for $5$ times, and average results on ADULT and COMPAS are reported in Table~\ref{tab:baseArc} and~\ref{tab:baseArc_compas}, respectively.

From the table, we observe that 
\begin{itemize}[leftmargin=*]
    \item Compared with the base classifiers, integrating {\method} makes the accuracy drops a little bit, which is in consistent with observations in other work on fair models~\cite{yan2020fair} as the fairness regularizer usually drops the accuracy. However, the accuracy decrease is marginal. For example, for LR, the accuracy only drops by 2\%, which shows that we are still able to maintain high accuracy;
    \item Though the accuracy drops a little bit, the fairness in terms of $\Delta_{EO}$ and $\Delta_{DP}$ on three models improves significantly, even though the sensitive attributes are not observed. For instance, for LR, with the {\method} framework, $\Delta_{EO}$ drops by 58.5\% while the accuracy only drops by 2\%. In other words, we scarify a little bit of accuracy while significantly improves the fairness. %which demonstrates the effectiveness and flexibility of the proposed method.
\end{itemize}

These observations show that {\method} can benefit various machine learning models to achieve fairness while maintaining high accuracy when the sensitive attributes are unknown
% For example, for SVM, {\method} can half the loss on fairness while keep the same accuracy simultaneously. These results provide further evidence over the effectiveness of {\method}.  

\end{document}